\documentclass[letterpaper, 10 pt, conference]{ieeeconf}
\IEEEoverridecommandlockouts
\usepackage[T1]{fontenc}
\usepackage[utf8]{inputenc}
\usepackage{url}
\usepackage{amsmath}

\usepackage{amsthm}
\usepackage{amssymb}
\usepackage{graphicx}
\usepackage{cite}

\theoremstyle{plain}
\newtheorem{thm}{\protect\theoremname}
\theoremstyle{remark}
\newtheorem{rem}[thm]{\protect\remarkname}

\providecommand{\remarkname}{Remark}
\providecommand{\theoremname}{Theorem}

\title{\LARGE \bf
  Orbital Stabilization and Time Synchronization of \\
  Unstable Periodic Motions in Underactuated Robots
}

\author{
  Maksim Surov
}

\begin{document}
  \maketitle
  \thispagestyle{empty}
  \pagestyle{empty}

  \begin{abstract}
    This paper presents a control methodology for achieving orbital stabilization
    with simultaneous time synchronization of periodic trajectories in
    underactuated robotic systems. The proposed approach extends the classical
    transverse linearization framework to explicitly incorporate time-desynchronization
    dynamics. To stabilize the resulting extended transverse dynamics,
    we employ a combination of time-varying LQR and sliding-mode control.
    The theoretical results are validated experimentally through the implementation
    of both centralized and decentralized control strategies on a group
    of six Butterfly robots.
  \end{abstract}

  \section{INTRODUCTION}

    The problem of trajectory tracking for underactuated robots has been
    addressed in a series of publications~\cite{Shiriaev-2005,Shiriaev-2010,Canudas-2002,Manchester,Navarro-2024,Romero-2022,Saetre-2021}.
    Most of these works focus on designing control algorithms for orbital
    stabilization, where the system state converges to a reference periodic
    trajectory up to a phase shift. Formulating the control objective
    in this way has enabled the development of algorithms that have demonstrated
    effectiveness in real-world applications \cite{Buss-2016,Freidovich-1,Freidovich-2,Surov-2015}.
    However, in some practical scenarios, orbital asymptotic stability
    alone may be insufficient. For example, in cooperative or synchronized
    tasks involving multiple underactuated robots, it may be necessary
    to ensure asymptotic stability of the full state rather than only
    the orbit itself, particularly when the robots share the same clock.

    A straightforward method for tracking a reference trajectory is based
    on linearization of the tracking error dynamics, followed by the design
    of an LQR for the resulting linear time-varying (LTV) system. This
    method is described in Chapter 12 of~\cite{Khalil} and has been
    shown to achieve asymptotic stability for small tracking errors, as
    demonstrated in experiments with a triple pendulum on a cart~\cite{Gluck-2013}.
    Compared to orbital stabilization methods, this approach is sensitive
    to initial time shifts, and the control system may lose stability
    if the robot becomes desynchronized.

    Alternative approaches for synchronization of closed orbits in underactuated
    robots involve modifications of orbital tracking algorithms to ensure
    synchronization between robots in a group~\cite{Shiriaev-2010,Chen-2015,Surov-2024,Xiang-2012,Jankuloski-2012}.
    For example, in~\cite{Shiriaev-2010}, the authors employ transverse
    linearization of the dynamics of a group of three robots to design
    a centralized control law that achieves synchronization. In~\cite{Jankuloski-2012}
    orbital stabilization together with synchronization is attained using
    the dynamic virtual holonomic constraints approach. In~\cite{Surov-2024},
    the authors propose an ad-hoc modification of the transverse--linearization
    approach, and demonstrate its effectiveness experimentally on the
    synchronization of two real robots.

    Our method for orbital stabilization with simultaneous time synchronization
    also represents a modification of orbital stabiliation feedback. It
    builds on the transverse linearization framework~\cite{Banaszuk,Shiriaev-2010,Shiriaev-2005}.
    For a given periodic trajectory, we augment the transverse dynamics
    with the dynamics of robot desynchronization, defined as the difference
    between the physical time and the reference time corresponding to
    the ``closest'' point on the trajectory. As we show, the linearization
    of this extended transverse dynamics takes the form of an LTV system,
    which can be stabilized using a combination of LQR and sliding-mode
    control, similarly to~\cite{Saetre-2021}. The resulting feedback
    law naturally decomposes into an orbital stabilization component and
    a synchronization component. The synchronization term is bounded,
    with the desynchronization variable entering through the signum function.
    This structure allows the control law to preserve the benefits of
    orbital stabilization while providing bounded corrective actions,
    even for large desynchronizations. 

    The remainder of the paper is organized as follows. Section~\ref{sec:Problem-Statement}
    formulates the problem of periodic trajectory tracking for a class
    of nonlinear control systems. Section~\ref{sec:Transverse-Linearization-Approac}
    briefly reviews the orbital-stabilization algorithm based on the transverse
    linearization approach. The main results are presented in Section~\ref{sec:Time-Synchronization},
    where the transverse dynamics are extended with a desynchronization
    variable and two methods for extended dynamics stabilization are proposed.
    The first method applies an LQR design to linearization of the extended
    transverse dynamics, while the second employs a sliding-mode control
    methodology. Section~\ref{sec:Experimental-Result} reports experimental
    results obtained on a group of six butterfly robots~\cite{Lynch-1998}.
    Concluding remarks are provided in Section~\ref{sec:Concluding-Remarks}.

  \section{PROBLEM STATEMENT}
  \label{sec:Problem-Statement}

    We consider a class of nonlinear control-affine systems of the form
    \begin{equation}
    \dot{x}=f\left(x\right)+g\left(x\right)u,\label{eq:affine-system}
    \end{equation}
    where $x\in\mathcal{X}$ is the system state and $u\in\mathbb{R}^{m}$
    is the control input. The state space $\mathcal{X}$ is a smooth $n$-dimensional
    manifold. For mechanical system $\mathcal{X}$ typically equals to
    the tangent bundle $TQ$ of the configuration manifold $Q$. The functions
    $f\left(x\right),$ $g\left(x\right)$ are assumed to be continuously
    differentiable on $\mathcal{X}$. 

    The system~(\ref{eq:affine-system}) is supposed to admit a smooth,
    nontrivial, $T$-periodic solution 
    \[
      x_{*}\left(t\right)=x_{*}\left(t+T\right),\quad u_{*}\left(t\right)=u_{*}\left(t+T\right)\quad\forall t\in\mathbb{R}
    \]
    with period $T>0$, satisfying the equation
    \[
      \frac{d}{dt}x_{*}\left(t\right)\equiv f\left(x_{*}\left(t\right)\right)+g\left(x_{*}\left(t\right)\right)u_{*}\left(t\right).
    \]
    The trajectory $x_{*}\left(t\right)$ is assumed to be free of equilibrium
    points ($\dot{x}_{*}\left(t\right)\ne0$ for all $t$) and self-intersections.
    The set of phase-space points associated with the trajectory is referred
    to as the orbit~\cite{Khalil} and is defined as
    \[
      \gamma\equiv\left\{ x\in\mathcal{X}\mid\exists t:x=x_{*}\left(t\right)\right\} .
    \]
    Under mild regularity assumptions, $\gamma$ forms a smooth embedded
    submanifold of $\mathcal{X}$. For sufficiently small $\epsilon>0$,
    we define the tubular neighborhood of the orbit by
    \[
      U_{\epsilon}\equiv\left\{ x\in\mathcal{X}\mid\mathrm{dist}\left(x,\gamma\right)<\epsilon\right\} ,
    \]
    where $\mathrm{dist}$ denotes the distance operator on $\mathcal{X}$.

    In the context of underactuated robots, the state $x$ often consists
    of generalized coordinates and velocities, $x\equiv\left(q;\dot{q}\right)\in TQ$,
    while the input $u$ represents generalized forces, with $1\leq\dim u<\dim q$.
    The periodic trajectory $x_{*}\left(t\right)$ thus describes a repetitive
    motion of the robot.

    A fundamental problem in robotics, as discussed in~\cite{Lefeber-2000,Siciliano-2008},
    is to design a feedback law $u=u\left(t,x\right)$, such that the
    reference trajectory $x_{*}\left(t\right)$ becomes a stable solution
    of the closed-loop system
    \begin{equation}
      \dot{x}=f\left(x\right)+g\left(x\right)u\left(t,x\right).\label{eq:closed-loop-system}
    \end{equation}
    Depending on the application, different notions of stability may be
    required. In the control of underactuated robots -- such as bipedal
    walking mechanisms~\cite{Westervelt-2007,Canudas-2002} -- the most
    common requirement is local asymptotic orbital stability, meaning
    that trajectories initiated within $U_{\epsilon}$ asymptotically
    converge to the orbit $\gamma$. In theory, a trajectory $x_{1}\left(t\right)$
    converges to the reference trajectory $x_{*}\left(t\right)$ up to
    a constant phase shift $h$, i.e. $x_{1}\left(t\right)-x_{*}\left(t-h\right)\to0$
    as $t\to\infty$. In practice, however, due to unmodeled dynamics
    and disturbances, the solutions $x_{1}\left(t\right)$ and $x_{*}\left(t\right)$
    may drift apart while $x_{1}\left(t\right)$ remains within $U_{\epsilon}$. 

    Another class of control problems seeks to construct a feedback law
    that guarantees asymptotic convergence of the tracking error $x-x_{*}\left(t\right)\to0$
    as $t\to\infty$. While this formulation eliminates phase desynchronization,
    there are relatively few ready-to-use algorithms available~\cite{Gluck-2013}.
    Moreover, existing methods that provide asymptotic tracking stability
    are often less robust than those designed for asymptotic orbital stabilization.

    To address these limitations, we propose a control method that achieves
    asymptotic stability by modifying an existing orbital stabilization
    algorithm based on the transverse linearization approach. For completeness
    of the exposition, we briefly recall the main steps of the transverse
    linearization framework before presenting our modification.

  \section{TRANSVERSE LINEARIZATION APPROACH}
  \label{sec:Transverse-Linearization-Approac}

  \subsection{Dynamics in Transverse Coordinates}

    One well-known approach to orbital stabilization is based on transverse
    linearization~\cite{Banaszuk,Shiriaev-2005,Shiriaev-2010,Surov-2020}.
    For completeness, we briefly recall the main steps of this approach.
    The key idea is to introduce a local coordinate transformation in
    a neighborhood of the reference periodic orbit. This transformation
    defines new coordinates
    \begin{align}
    \theta & =\pi\left(x\right)\in\left[0,T_{\theta}\right)\quad\text{and}\quad\xi=\alpha\left(x\right)\in\mathbb{R}^{n-1},\label{eq:transverse-coordinates}
    \end{align}
    where $\theta$ represents the phase obtained by projecting the point
    $x$ onto the orbit $\gamma$, and $\xi$ collects $n-1$ transverse
    deviations from $\gamma$. The transformation~(\ref{eq:transverse-coordinates})
    is assumed to be a diffeomorphism within $U_{\epsilon}$, which in
    particular postulates the existence of the inverse mapping
    \[
    x=\beta\left(\xi,\theta\right).
    \]
    The projection operator $\pi$ is assumed to satisfy the monotonicity
    condition 
    \begin{equation}
    \frac{d}{dt}\pi\left(x_{*}\left(t\right)\right)=\frac{\partial\pi\left(x_{*}\left(t\right)\right)}{\partial x}\dot{x}_{*}\left(t\right)\geq\mathrm{const}>0\label{eq:monotonicity-of-theta}
    \end{equation}
    for all $t$. This condition implies that $\theta_{*}\left(t\right)\equiv\pi\left(x_{*}\left(t\right)\right)$
    is strictly increasing for $t\in\left[0,T\right)$ and therefore can
    be extended continuously beyond one period as $\theta_{*}\left(t+kT\right)\equiv\theta_{*}\left(t\right)+kT_{\theta}$
    for all $t\in\left[0,T\right)$ and $k\in\mathbb{Z}$. Continuity
    and monotonicity of $\theta_{*}\left(t\right)$ guarantee the existence
    of the inverse mapping $t=\theta_{*}^{-1}\left(\theta\right)$, which
    allows alternative parametrization of the reference signals
    \begin{align*}
    x_{*}\left(\theta\right) & \equiv x_{*}\left(\theta_{*}^{-1}\left(\theta\right)\right),\quad u_{*}\left(\theta\right)\equiv u_{*}\left(\theta_{*}^{-1}\left(\theta\right)\right).
    \end{align*}
    Following standard abuse of notation in the mathematical literature,
    we use the same symbol $x_{*}$ for both the original function $x_{*}\left(t\right)$
    and its reparametrization $x_{*}\left(\theta\right)\equiv\left(x_{*}\circ\theta_{*}^{-1}\right)\left(\theta\right)$.
    In this parametrization, $x_{*}\left(\theta\right)$ represents a
    smooth curve in $\mathcal{X}$, and the operator $\pi\left(x\right)$
    projects a point $x\in U_{\epsilon}$ onto the curve and yields the
    corresponding parameter value $\theta$. Then, the original system~(\ref{eq:affine-system})
    can be expressed in the new coordinates $\theta,\xi$ as 
    \begin{align}
    \dot{\xi} & =f_{\xi}\left(\xi,\theta\right)+g_{\xi}\left(\xi,\theta\right)w\label{eq:dot_xi}\\
    \dot{\theta} & =f_{\theta}\left(\xi,\theta\right)+g_{\theta}\left(\xi,\theta\right)w\label{eq:dot_theta}
    \end{align}
    where the new control input is defined by 
    \[
    w\equiv u-u_{*}\left(\theta\right),
    \]
    and the functions $f_{\xi}\left(\cdot\right),f_{\theta}\left(\cdot\right),g_{\xi}\left(\cdot\right),g_{\theta}\left(\cdot\right)$
    are given by
    \begin{align}
    \left(\begin{array}{c}
    f_{\xi}\left(\xi,\theta\right)\\
    f_{\theta}\left(\xi,\theta\right)
    \end{array}\right) & \equiv\left(\frac{\partial\beta\left(\xi,\theta\right)}{\partial\left(\xi,\theta\right)}\right)^{-1}f\left(\beta\left(\xi,\theta\right)\right)\label{eq:new_f_g}\\
    & \quad+\left(\frac{\partial\beta\left(\xi,\theta\right)}{\partial\left(\xi,\theta\right)}\right)^{-1}g\left(\beta\left(\xi,\theta\right)\right)u_{*}\left(\theta\right),\nonumber \\
    \left(\begin{array}{c}
    g_{\xi}\left(\xi,\theta\right)\\
    g_{\theta}\left(\xi,\theta\right)
    \end{array}\right) & \equiv\left(\frac{\partial\beta\left(\xi,\theta\right)}{\partial\left(\xi,\theta\right)}\right)^{-1}g\left(\beta\left(\xi,\theta\right)\right).\nonumber 
    \end{align}
    We notice, that the jacobian matrix $\frac{\partial\beta\left(\xi,\theta\right)}{\partial\left(\xi,\theta\right)}$
    of the inverse transform $x=\beta\left(\xi,\theta\right)$ is invertible
    for all sufficiently small $\xi$, due to the presence of local diffeomorphism
    $\left(\xi,\theta\right)\simeq x$. 

    Dividing the right-hand side of~(\ref{eq:dot_xi}) by that of (\ref{eq:dot_theta}),
    yields the transverse dynamics in a general nonlinear form:
    \begin{align}
    \frac{d\xi}{d\theta} & =\frac{f_{\xi}\left(\xi,\theta\right)+g_{\xi}\left(\xi,\theta\right)w}{f_{\theta}\left(\xi,\theta\right)+g_{\theta}\left(\xi,\theta\right)w}\equiv\phi\left(\xi,\theta,w\right).\label{eq:transverse-dynamics-nonlin}
    \end{align}
    The denominator in~(\ref{eq:transverse-dynamics-nonlin}) is strictly
    positive for sufficiently small $\xi$ and $w$ due to the monotonicity
    condition~(\ref{eq:monotonicity-of-theta}). Consequently, within
    a sufficiently small tube $U_{\epsilon}$, the dynamics (\ref{eq:transverse-dynamics-nonlin})
    are well defined and equivalent to the original system~(\ref{eq:affine-system}).
    Thus, the problem of stabilizing the orbit $\gamma$ is transformed
    into the problem of stabilizing the trivial solution $\xi=0$ of the
    transverse dynamics~(\ref{eq:transverse-dynamics-nonlin}). 

  \subsection{Orbital Stabilization Control Design}
  \label{sub:orbital_stabilization_control_design}

    A common approach for stabilizing~(\ref{eq:transverse-dynamics-nonlin})
    is to linearize the right-hand side with respect to $\xi$ and $w$:
    \begin{equation}
    \frac{d\xi}{d\theta}=A\left(\theta\right)\xi+B\left(\theta\right)w+o\left(\theta,\xi,w\right),\label{eq:transverse-linearization}
    \end{equation}
    where the matrix functions
    \begin{align*}
    A\left(\theta\right) & \equiv\left.\frac{\partial\phi\left(\xi,\theta,w\right)}{\partial\xi}\right|_{\xi=0,w=0}\in\mathbb{R}^{\left(n-1\right)\times\left(n-1\right)}\quad\text{and}\\
    B\left(\theta\right) & \equiv\left.\frac{\partial\phi\left(\xi,\theta,w\right)}{\partial w}\right|_{\xi=0,w=0}\in\mathbb{R}^{\left(n-1\right)\times m}
    \end{align*}
    are $T_{\theta}$-periodic, and $o\left(\xi,\theta,w\right)$ collects
    bilinear, quadratic, and higher-order terms in $\xi$ and $w$. As
    shown in Theorem~2 in~\cite{Yakubovich-1986}, controllability over
    period of an LTV system with periodic coefficients
    \[
    \frac{d\bar{\xi}}{d\theta}=A\left(\theta\right)\bar{\xi}+B\left(\theta\right)w
    \]
    implies the existence of an exponentially stabilizing feedback $w\left(\bar{\xi},\theta\right)=K\left(\theta\right)\bar{\xi}$
    with a $T_{\theta}$-periodic gain matrix $K\left(\theta\right)$.
    This feedback $w\left(\bar{\xi},\theta\right)$ allows construction
    of the state feedback for the original nonlinear system~(\ref{eq:affine-system})
    as
    \begin{equation}
    u\left(x\right)=u_{*}\left(\pi\left(x\right)\right)+K\left(\pi\left(x\right)\right)\alpha\left(x\right).\label{eq:transverse-feedback}
    \end{equation}

    Using aruments in~\cite{Leonov-2006}, one can conclude that the
    feedback law~(\ref{eq:transverse-feedback}) guarantees local asymptotic
    orbital stability of the trajectory $x_{*}\left(t\right)$. However,
    it does not ensure asymptotic stability of the trajectory itself.

  \section{TIME SYNCHRONIZATION ALGORITHM}
  \label{sec:Time-Synchronization}

  \subsection{Extended Transverse Dynamics}

    Below, we consider an extension of the transverse dynamics~(\ref{eq:transverse-dynamics-nonlin})
    that incorporates the dynamics of the robot’s time desynchronization,
    expressed in terms of the variables $\xi,$ and $\theta$.

    Let $x\in U_{\epsilon}$ be the state of the system at time $t$.
    Using the projection operator $\pi\left(x\right)$ and the alternative
    parametrization $\theta_{*}\left(t\right)$, the robot's self-time
    is defined as
    \begin{equation}
    \tau\equiv\theta_{*}^{-1}\left(\pi\left(x\right)\right)+\nu T,\label{eq:tau-def}
    \end{equation}
    where $\nu\in\mathbb{N}$ counts the total number of periods completed
    by the robot. The controller is assumed to maintain an internal state
    that tracks and updates $\nu$. The robot lateness is then defined
    as
    \begin{equation}
    h\equiv t-\tau.\label{eq:h-def}
    \end{equation}
    We note that, while $t$ is ordinarily considered the physical (absolute)
    time, after rewriting the dynamics in terms of the variable $\theta$,
    $t$ effectively becomes a dynamic variable. Its evolution along a
    solution $\left(\xi\left(\theta\right),w\left(\theta\right)\right)$
    of the dynamical system~(\ref{eq:transverse-dynamics-nonlin}) is
    determined via~(\ref{eq:dot_theta}) as
    \begin{equation}
    t=t_{0}+\int_{\theta_{0}}^{\theta}\frac{1}{f_{\theta}\left(\xi\left(\varepsilon\right),\varepsilon\right)+g_{\theta}\left(\xi\left(\varepsilon\right),\varepsilon\right)w\left(\varepsilon\right)}d\varepsilon,\label{eq:t-integral}
    \end{equation}
    so that the dynamics of $t$ now depends on the behavior of new variables. 

    Using the expressions in~(\ref{eq:dot_theta},\ref{eq:transverse-dynamics-nonlin},\ref{eq:tau-def},\ref{eq:h-def}),
    the extended transverse dynamics can be written explicitly in terms
    of the variables $\xi,\theta$ and $w$ as
    \begin{align}
    \frac{d\xi}{d\theta} & =\frac{f_{\xi}\left(\xi,\theta\right)+g_{\xi}\left(\xi,\theta\right)w}{f_{\theta}\left(\xi,\theta\right)+g_{\theta}\left(\xi,\theta\right)w}\label{eq:aug-nonlin-transdyn}\\
    \frac{dh}{d\theta} & =\frac{1}{f_{\theta}\left(\xi,\theta\right)+g_{\theta}\left(\xi,\theta\right)w}-\frac{1}{\theta'_{*}\left(\theta_{*}^{-1}\left(\theta\right)\right)},\nonumber 
    \end{align}
    where the notation $\theta'_{*}\left(\tau\right)\equiv\frac{d\theta{}_{*}\left(\tau\right)}{d\tau}$
    used. It is straighforward to verify, that the system~(\ref{eq:aug-nonlin-transdyn})
    possesses the trivial solution $\xi=0,$ $h=0,$ $w=0,$ which corresponds
    to the referece solution $\left(x_{*}\left(t\right),u_{*}\left(t\right)\right)$
    of the original nonlinear system~(\ref{eq:affine-system}). Then,
    the control objective can be formulated as follows: design a feedback
    law $w\left(\xi,h,\theta\right)$ for the system~(\ref{eq:aug-nonlin-transdyn}),
    which asymptotically stabilizes the trivial solution $\xi=0,h=0$.

  \subsection{LQR Approach for Controller Design}

    As seen, the system~(\ref{eq:aug-nonlin-transdyn}) represents a
    general nonlinear time-varying system. A common first step toward
    stabilizing its trivial solution is to design a feedback controller
    for the linearization of the system, which then serves as a basis
    for handling the nonlinear dynamics~\cite{Khalil}. The approximate
    linearization of the right-hand side of~(\ref{eq:aug-nonlin-transdyn})
    yields the LTV system
    \begin{align}
    \frac{d\xi}{d\theta} & =A_{\xi}\left(\theta\right)\xi+B_{\xi}\left(\theta\right)w+o_{\xi}\left(\theta,\xi,w\right)\label{eq:aug-lin-transdyn}\\
    \frac{dh}{d\theta} & =A_{h}\left(\theta\right)\xi+B_{h}\left(\theta\right)w+o_{w}\left(\theta,\xi,w\right)\nonumber 
    \end{align}
    with $T_{\theta}$-periodic matrices $A_{\xi}\left(\theta\right),A_{h}\left(\theta\right),B_{\xi}\left(\theta\right),B_{h}\left(\theta\right)$.
    The functions $o_{\xi}\left(\theta,\xi,w\right)$ and $o_{w}\left(\theta,\xi,w\right)$
    collect bilinear, quadratic and higher order terms in $\xi$ and $w$.
    Controllability over one period of this LTV system, as discussed in~\cite{Yakubovich-1986},
    guarantees the existence of an exponentially stabilizing feedback
    of the form $w\left(\theta,\xi,h\right)=K_{\xi}\left(\theta\right)\xi+K_{h}\left(\theta\right)h$
    with periodic gains $K_{\xi}\left(\theta\right)$ and $K_{h}\left(\theta\right)$,
    which can be computed using standard numerical methods~\cite{Gusev-2010}.
    This feedback can then be translated into a control law for the original
    nonlinear system~(\ref{eq:affine-system}) as
    \begin{align}
    u\left(t,x\right)= & u_{*}\left(\theta\right)+K_{\xi}\left(\theta\right)\alpha\left(x\right)+K_{h}\left(\theta\right)\left(t-\tau\right),\label{eq:timesync-lqr}
    \end{align}
    where $\theta$ and $\tau$ are computed from $x$ as described above,
    and $t$ denotes the physical time. 

    A complete proof of asymptotic stability of the reference trajectory
    $x_{*}\left(t\right)$ for the closed loop nonlinear system is beyond
    the scope of the present publication. Instead, we refer to the standard
    arguments presented in Theorem 4.13 of~\cite{Khalil}.

    It is important to note, that the control law~(\ref{eq:timesync-lqr})
    guarantees local asymptotic stability only. Specifically, for any
    initial state $x_{0}$ within a small tubular neighborhood $U_{\epsilon}$
    and with a small initial time deviation $\left|t_{0}-\theta_{*}^{-1}\left(\pi\left(x_{0}\right)\right)\right|$
    the corresponding solution of the closed-loop system asymptotically
    converges to $x_{*}\left(t\right)$. For larger initial time deviations,
    stability may be compromised even when $x_{0}\in\gamma$. Consequently,
    strict time synchronization can potentially reduce the robustness
    advantages of orbital stabilization. To address this issue, the control
    law can be modified by saturating the lateness term:
    \begin{align*}
    u & =u_{*}\left(\theta\right)+K_{\xi}\left(\theta\right)\alpha\left(x\right)+K_{h}\left(\theta\right)\mathrm{sat}\left(t-\tau,-h_{\max},h_{\max}\right),
    \end{align*}
    where $\mathrm{sat}\left(x,x_{\min},x_{\max}\right)$ denotes the saturation of $x$ to the range $[x_{\min},x_{\max}]$.
    While this modification is observed to improve practical performance,
    it may, in theory, compromise the stability of the closed-loop system
    and thus warrants fundamental investigation. We therefore propose
    an alternative control scheme based on sliding-mode control, where
    the synchronization term remains bounded regardless of the magnitude
    of the lateness $h$.

  \subsection{Sliding Mode Approach for Controller Design}

    \label{sec:Sliding-Mode-Approach}

    The method developed below is based on designing a sliding-mode feedback
    in which the lateness variable $h$ enters the argument of the signum
    function. In this way, the controller accelerates the motion whenever
    the system lags (regardless of the magnitude of the delay) and decelerates
    otherwise. Although the design still relies on the analysis of the
    linearized transverse dynamics~(\ref{eq:aug-lin-transdyn}), it adopts
    the alternative stabilization approach for LTV systems proposed in~\cite{Saetre-2021}.

    For clarity, the following theorem states the sliding-mode control
    law for the case of a scalar input $u\in\mathbb{R}$. The multi-input
    version can be formulated analogously, but its full statement and
    proof are omitted here for brevity.
    \begin{thm}
    \label{thm:sliding-control} Consider the LTV system 
    \begin{align}
    \dot{x} & =A_{x}\left(t\right)x+B_{x}\left(t\right)u\label{eq:extended-ltv-dynamics}\\
    \dot{z} & =A_{z}\left(t\right)x+B_{z}\left(t\right)u\nonumber 
    \end{align}
    where $x\in\mathbb{R}^{n},z\in\mathbb{R}$ are state variables and
    $u\in\mathbb{R}$ is the control input. Assume the following:
    \begin{enumerate}
      \item The matrix functions $A_{x}\left(t\right),$ $A_{z}\left(t\right),$
      $B_{x}\left(t\right),$ $B_{z}\left(t\right)$ are continuous and
      $T$-periodic, $T>0$.
      \item The homogeneous subsystem $\dot{x}=A_{x}\left(t\right)x$ is uniformly
      exponentially stable.
    \end{enumerate}
    Let $n\left(t\right)\in\mathbb{R}^{n}$ denote the unique $T$-periodic
    solution of the equation
    \[
    \dot{n}=-A_{x}^{\top}\left(t\right)n-A_{z}^{\top}\left(t\right)
    \]
    and suppose that the function 
    \[
    b\left(t\right)\equiv B_{x}^{\top}\left(t\right)n\left(t\right)+B_{z}\left(t\right)\in\mathbb{R}
    \]
    satisfies
    \[
    \left|b\left(t\right)\right|\geq b_{\min}>0\quad\text{for all}\quad t.
    \]
    Then, for any $k>0$, the feedback law
    \begin{equation}
    u\left(t,x,z\right)=-\frac{k}{b\left(t\right)}\mathrm{sign}\left(n^{\top}\left(t\right)x+z\right)\label{eq:sliding-control}
    \end{equation}
    renders the equilibrium $x=0,z=0$ of~(\ref{eq:extended-ltv-dynamics})
    globally asymptotically stable.

    \end{thm}
    \begin{proof}
    The control law~(\ref{eq:sliding-control}) is a sliding-mode control
    defined on the time-varying sliding surface 
    \[
    s\left(t,x,z\right)=n^{\top}\left(t\right)x+z=0
    \]
    in the extended phase space $\left(t,x,z\right)$. Differentiating
    $s$ along system trajectories gives
    \[
    \dot{s}=\dot{n}^{\top}\left(t\right)x+n^{\top}\left(t\right)\dot{x}+\dot{z}.
    \]
    Under the control~(\ref{eq:sliding-control}), this becomes 
    \[
    \dot{s}=-k\cdot\mathrm{sign}s,
    \]
    so $s\left(t\right)$ reaches zero at finite time (at most $\frac{\left|s\left(0\right)\right|}{k}$).
    On the surface $s\left(t,x,z\right)=0$ the equivalent control~\cite{Utkin}
    is obtained from 
    \[
    \dot{s}=\frac{\partial s\left(t,x,z\right)}{\partial t}+\frac{\partial s\left(t,x,z\right)}{\partial x}\dot{x}+\frac{\partial s\left(t,x,z\right)}{\partial z}\dot{z}=0,
    \]
    which yields $u_{eq}=0$ whenever $b\left(t\right)\ne0$. Substituting
    the $u_{eq}$ into~(\ref{eq:extended-ltv-dynamics}) gives the reduced
    dynamics
    \begin{align*}
    \dot{x} & =A_{x}\left(t\right)x,\quad z=n^{\top}\left(t\right)x,
    \end{align*}
    which is uniformly exponentially stable by assumption~2. Because
    $n\left(t\right)$ is bounded, continuous and $T$-periodic, $z\left(t\right)=n^{\top}\left(t\right)x\left(t\right)\to0$
    whenever $x\left(t\right)\to0$. Finite--time convergence of $s$
    together with exponential stability of the in-manifold dynamics implies
    global asymptotic stability of the equilibrium $x=0,z=0$. 
    \end{proof}
    \begin{rem}
    While $s\ne0$, the $x$-subsystem is subjected to the periodic disturbance
    $\pm\frac{kB_{x}\left(t\right)}{b\left(t\right)}.$ Because the unforced
    system $\dot{x}=A_{x}\left(t\right)x$ is uniformly exponentially
    stable, the forced system converges to a bounded steady state whose
    amplitude of order $O(k)$. Consequently, for every sufficiently small
    $k>0$ and any prescribed $\epsilon>0$, there exists $\delta>0$
    such that every trajectory of the original nonlinear closed-loop system
    starting within $U_{\delta}$ remains inside $U_{\epsilon}$ for all
    $t\geq0$. This observation indicates that the closed loop nonlinear
    system stays close to the orbit $\gamma$ while synchronization term
    is converging.
    \end{rem}
    As can be seen, Theorem~\ref{thm:sliding-control} imposes relatively
    strong conditions on the LTV system. In particular, the dynamics~(\ref{eq:aug-lin-transdyn})
    generally do not satisfy the stability assumption for the $\xi$-subsystem.
    Consequently, Theorem~\ref{thm:sliding-control} cannot be applied
    directly to~(\ref{eq:aug-lin-transdyn}). To address this limitation,
    we propose the following algorithm.
    \begin{enumerate}
      \item Repeat the steps of construction orbitally stabilizing feedback
        mentioned in Section~\ref{sub:orbital_stabilization_control_design}.
        This yields gain matrix $K(\theta)$, which stabilizes $\xi$-subsystem. Then, the control $w$
        for the system~(\ref{eq:aug-lin-transdyn}) is decomposed as 
        \[
          w = K\left(\theta\right)\xi+w_{s},
        \]
        where the term $w_{s}$ is a new control input dedicated to stabilizing the lateness variable
        $h$. Under this control, the system~(\ref{eq:aug-lin-transdyn})
        becomes
        \begin{align}
          \frac{d\xi}{d\theta} & =\left(A_{\xi}\left(\theta\right)+B_{\xi}\left(\theta\right)K\left(\theta\right)\right)\xi+B_{\xi}\left(\theta\right)w_{s}\label{eq:control-transform-2}\\
          \frac{dh}{d\theta} & =\left(A_{h}\left(\theta\right)+B_{h}\left(\theta\right)K\left(\theta\right)\right)\xi+B_{h}\left(\theta\right)w_{s}.\nonumber 
        \end{align}

      \item Define $n\left(\theta\right)$ as the unique $T_{\theta}$-periodic
        solution of
        \begin{align}
          \frac{dn}{d\theta} & =-\left(A_{\xi}\left(\theta\right)+B_{\xi}\left(\theta\right)K\left(\theta\right)\right)^{\top}n\nonumber \\
          & \quad-\left(A_{h}\left(\theta\right)+B_{h}\left(\theta\right)K\left(\theta\right)\right)^{\top},\label{eq:sliding-surf-vec}
        \end{align}
        and verify that $n^{\top}\left(\theta\right)B_{\xi}\left(\theta\right)+B_{h}\left(\theta\right)\ne0$
        for all $\theta$.

      \item Design the control $w_s$ applying Theorem~\ref{thm:sliding-control}
        to~(\ref{eq:control-transform-2}). 
        Then, the resulting control for system~(\ref{eq:affine-system}) becomes
        \begin{equation}
          u\left(t,x\right)=u_{*}\left(\theta\right)+K\left(\theta\right)\xi-\frac{k\cdot\mathrm{sign}\left(n^{\top}\left(\theta\right)\xi+h\right)}{n^{\top}\left(\theta\right)B_{\xi}\left(\theta\right)+B_{h}\left(\theta\right)},\label{eq:sliding-nonlin-control}
        \end{equation}
        where $\xi,\theta,h$ are computed from $x$ and $t$ as described above. 
    \end{enumerate}
    As evident, the feedback~(\ref{eq:sliding-nonlin-control}) naturally
    decomposes into three components: the feedforward term $u_{*}\left(\theta\right)$,
    the orbital stabilization term $K\left(\theta\right)\xi$ and the
    time synchronization term $-\frac{k\cdot\mathrm{sign}\left(n^{\top}\xi+h\right)}{n^{\top}B_{\xi}+B_{h}}$.
    The synchronization term is bounded by 
    \[
    k\max_{\theta}\left|n^{\top}\left(\theta\right)B_{\xi}\left(\theta\right)+B_{h}\left(\theta\right)\right|^{-1}
    \]
    and can be made arbitrarily small by selecting a sufficiently small
    $k>0$. Selecting $k=0$ reduces the controller to standard orbital
    stabilization. This allows a tunable trade-off between orbit tracking
    performance and the rate of lateness decay.

  \section{EXPERIMENTAL RESULT}
  \label{sec:Experimental-Result}

    The experimental platform is the Butterfly robot~\cite{Lynch-1998},
    a planar mechanical system consisting of a figure-eight--shaped link
    mounted to a fixed base through a single revolute joint. The link
    rotates in the vertical plane and is driven by a DC motor. A small
    spherical ball is placed on the link and can roll passively along
    its surface without direct actuation. The link's angular position
    is measured by an encoder, while a vision system estimates the ball's
    position in real time. This apparatus is used to validate nonprehensile
    manipulation algorithms, including tasks such as transporting the
    ball between prescribed equilibrium points and sustaining stable periodic
    motions of the ball~\cite{Lynch-1999}.

    Assuming ideal rolling conditions and perfect spherical symmetry of
    the ball, together with the mechanism's rotational symmetry, the configuration
    space is topologically a torus. We parameterize the configuration
    by two angles, $\vartheta$ and $\varphi$. The angle $\vartheta$
    denotes the orientation of the link relative to the base, and $\varphi$
    specifies the ball's position in a polar coordinate frame attached
    to the link (Fig.~\ref{fig:butterfly-gen-coords}). The robot dynamics
    are given by the Euler-Lagrange equations in the form 
    \[
    M\left(q\right)\ddot{q}+C\left(q,\dot{q}\right)\dot{q}+G\left(q\right)=Bu,\quad q\equiv\left(\vartheta,\varphi\right),\quad u\in\mathbb{R}
    \]
    where $u\in\mathbb{R}$ is the torque applied to the link's revolute
    joint. The matrix $M\left(q\right)\in\mathbb{R}^{2\times2}$ is invertible,
    the matrix $B\equiv\left[1,0\right]^{\top}$. The exact expressions
    for the matrix functions $M\left(q\right),$ $C\left(q\right),$ $G\left(q\right)$
    and the physical parameters are given in~\cite{Surov-2015}.

    \begin{figure}
    \begin{centering}
    \includegraphics[width=8.5cm]{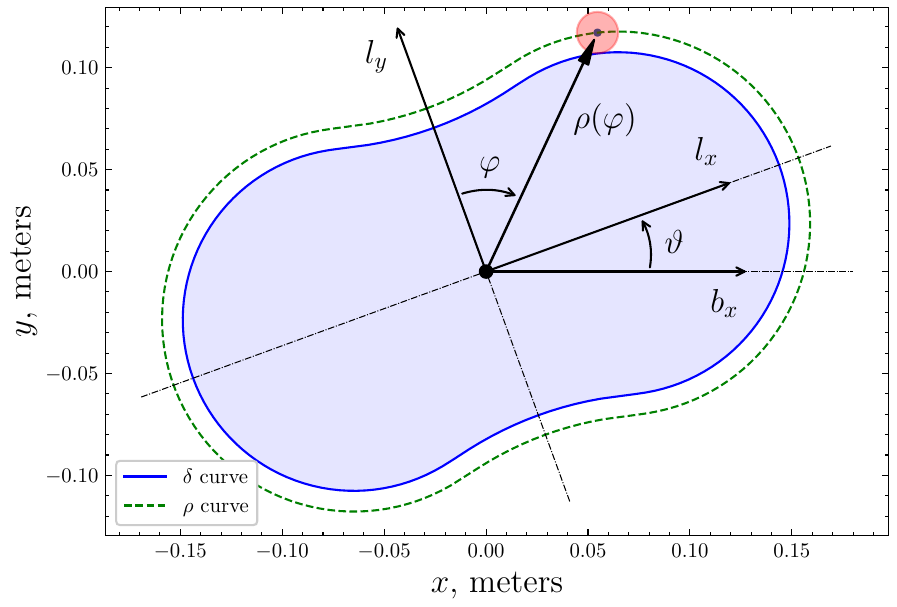}
    \par\end{centering}
    \caption{ Generalized coordinates of the Butterfly robot. The angle $\vartheta$
    defines the orientation of the link, while the angle $\varphi$ specifies
    the position of the ball.}\label{fig:butterfly-gen-coords}
    \end{figure}

    Because $M\left(q\right)$ is invertible, the dynamics can also be
    written in the control-affine form~(\ref{eq:affine-system}) with
    $x\equiv\left(\vartheta,\varphi,\dot{\vartheta},\dot{\varphi}\right)^{\top}$
    and
    \begin{align}
    f\left(x\right) & =\left(\begin{array}{c}
    \dot{q}\\
    -M^{-1}\left(q\right)\left(C\left(q,\dot{q}\right)\dot{q}+G\left(q\right)\right)
    \end{array}\right),\label{eq:butterfly-dyn}\\
    g\left(x\right) & =\left(\begin{array}{c}
    0\\
    M^{-1}\left(q\right)B\left(q\right)
    \end{array}\right).\nonumber 
    \end{align}
    As demonstrated in~\cite{Surov-2015}, the robot can execute one-directional
    periodic motions, in which the ball continuously rolls along the link
    in the same direction and the system returns to its initial configuration
    due to the topological properties of the configuration space. After
    one period, both the link orientation and the ball position advance
    by $2\pi$. We use this periodic trajectory, see Fig.~\ref{fig:butterfly-ref-traj},
    as an illustrative example to demonstrate the algorithm proposed in
    Section~\ref{sec:Sliding-Mode-Approach}.

    \begin{figure}
    \begin{centering}
    \includegraphics[width=8.5cm]{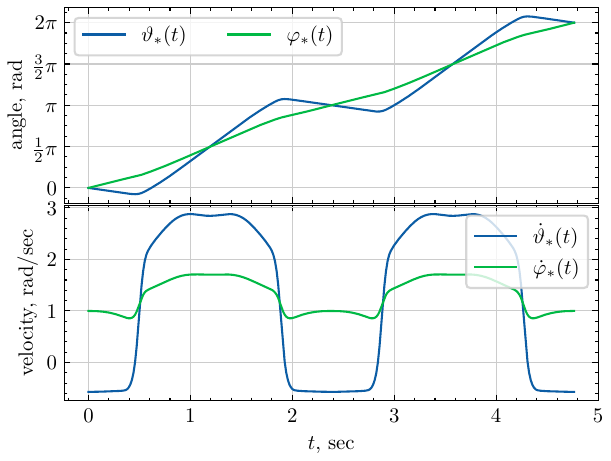}
    \par\end{centering}
    \caption{ Reference periodic trajectory of the Butterfly robot. The figure
    shows the generalized coordinates $\vartheta_{*}\left(t\right),\varphi_{*}\left(t\right)$
    and the corresponding generalized velocities $\dot{\vartheta}_{*}\left(t\right),\dot{\varphi}_{*}\left(t\right)$.}\label{fig:butterfly-ref-traj}
    \end{figure}

  \subsection{Transverse Coordinates}

    We apply the control law~(\ref{eq:sliding-nonlin-control}) to stabilize
    the trajectory $x_{*}\left(t\right)$ shown in Fig.~\ref{fig:butterfly-ref-traj}.
    To this end, let us define all the transformations. 

    As observed, the coordinate $\varphi_{*}\left(t\right)$ is strictly
    increasing, which motivates defining the projection operator
    \begin{equation}
    \theta=\pi\left(x\right)\equiv\varphi.\label{eq:butterfly-proj}
    \end{equation}
    Using this, the reference signals are reparametrized with respect
    to $\theta$ as 
    \begin{align*}
    & \vartheta_{*}\left(\theta\right)\equiv\vartheta_{*}\left(\varphi_{*}^{-1}\left(\theta\right)\right),\quad\dot{\vartheta}_{*}\left(\theta\right)\equiv\dot{\vartheta}_{*}\left(\varphi_{*}^{-1}\left(\theta\right)\right)\\
    & \dot{\varphi}_{*}\left(\theta\right)\equiv\dot{\varphi}_{*}\left(\varphi_{*}^{-1}\left(\theta\right)\right),\quad u_{*}\left(\theta\right)\equiv u_{*}\left(\varphi_{*}^{-1}\left(\theta\right)\right).
    \end{align*}
    For this trajectory the period $T$ equals $4.77$ sec, which corresponds
    $T_{\theta}=2\pi$. This parametrization allows us to define the control
    input transformation and the transverse coordinates as
    \begin{equation}
    w\equiv u-u_{*}\left(\theta\right),\quad\xi=\alpha\left(x\right)\equiv\left(\begin{array}{c}
    \vartheta-\vartheta_{*}\left(\varphi\right)\\
    \dot{\vartheta}-\dot{\vartheta}_{*}\left(\varphi\right)\\
    \dot{\varphi}-\dot{\varphi}_{*}\left(\varphi\right)
    \end{array}\right).\label{eq:butterfly-trans}
    \end{equation}
    The inverse transformation $x=\beta\left(\xi,\theta\right)$ can be
    expressed in a straightforward manner.

    The linearization of the extended transverse dynamics~(\ref{eq:aug-lin-transdyn})
    was performed following the formulas~(\ref{eq:new_f_g},\ref{eq:aug-nonlin-transdyn}),
    using the robot dynamics~(\ref{eq:butterfly-dyn}) and the transformations~(\ref{eq:butterfly-proj},\ref{eq:butterfly-trans}).
    Direct computation of the controllability Gramian over the interval
    $\left[0,2\pi\right]$ confirms that the linear system is controllable
    throughout the interval. This fact allows the construction of an LQR
    for the subsystem $\left(A_{\xi}\left(\theta\right),B_{\xi}\left(\theta\right)\right)$
    with weighting matrices $Q=\mathrm{diag\left(20,1,1\right)}$ and
    $R=2\cdot10^{4}$. The stabilizing feedback gains $K\left(\theta\right)$
    are found using the LMI approach~\cite{Gusev-2010}. 

    The next step of the algorithm is the construction of the synchronization
    term, which involves computing the periodic solution $n\left(\theta\right)$
    of equation (\ref{eq:sliding-surf-vec}). To obtain this solution,
    we note that reversing the argument, $\theta\mapsto-\theta$, renders
    equation (\ref{eq:sliding-surf-vec}) exponentially stable. This implies
    that all solutions of the modified equation converge to the periodic
    function $n\left(-\theta\right)$. Consequently, the function $n\left(\theta\right)$
    can be computed by numerically integrating the modified equation with
    arbitrary initial conditions over several periods, until the resulting
    function $n\left(-\theta\right)$ becomes effectively periodic. Upon
    completing this step of the algorithm, all the functions required
    to construct the feedback law~(\ref{eq:sliding-nonlin-control})
    are available.

  \subsection{Asymptotic Stabilization of the Robot Trajectory}

    The algorithm described above was tested on the physical robot. Fig.~\ref{fig:lateness-converge}
    shows the evolution of the lateness $h$, the transverse coordinates
    $\xi$ and the rate of projection variable $\dot{\theta}$ as functions of physical time $t$. 
    As observed, the lateness converges toward zero at an approximately constant rate,
    despite an initial lateness of $15$\,s, which is large relative to
    the trajectory period. The coordinates $\xi$ remain bounded throughout
    the experiment. They do not completely vanish due to external disturbances,
    such as unmodeled dissipative forces arising from non-ideal rolling.
    Furthermore, the signals received by the control system from the vision
    system are subject to unknown delays, which are approximately $15$
    ms but may vary. The graph of $\dot{\theta}$ shows that the projection
    parameter's rate increases above the nominal value when it lags behind
    the reference trajectory and decreases as it catches up.

    Fig.~\ref{fig:phase-converg} illustrates the convergence of the
    deviation $x-x_{*}\left(t\right)$. All components converge close
    to the origin, remaining bounded within a small range due to unmodeled
    dynamics and sensor noise, confirming the practical asymptotic stability
    of the closed-loop system.

    \begin{figure}
    \begin{centering}
      \includegraphics[width=8.5cm]{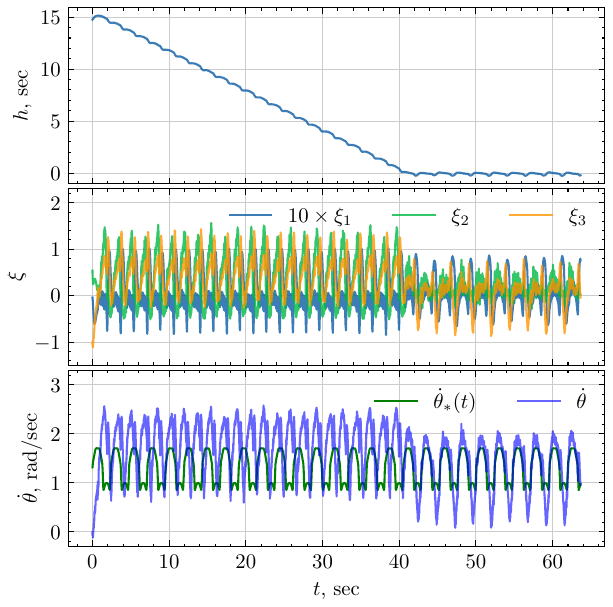}
      \par\end{centering}
      \caption{
        Lateness of the robot $h$, transverse coordinates $\xi$ and the rate
        of change of the projection variable $\dot{\theta}$.
      }
      \label{fig:lateness-converge}
    \end{figure}

    \begin{figure}
      \begin{centering}
      \includegraphics[width=8.5cm]{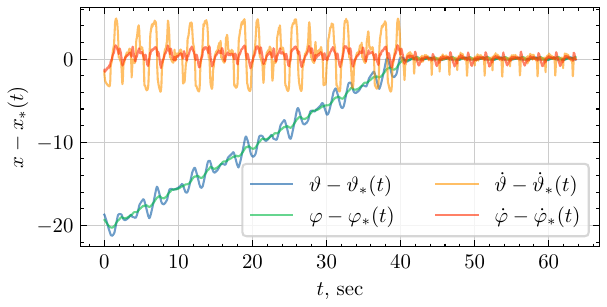}
      \par\end{centering}
      \caption{ Transient behavior of the tracking error $x-x_{*}\left(t\right)$.}
      \label{fig:phase-converg}
    \end{figure}

    The control law~(\ref{eq:sliding-nonlin-control}) can also be applied
    to synchronize a group of robots provided they share a common
    clock. Experimental results for the synchronization of two robots
    are shown in Fig.~\ref{fig:sync-all-to-global-time}. Immediately
    after launch the first robot lags behind the schedule, while the second
    robot runs ahead of it. The robots then converge to the same scheduled
    trajectory $x_{*}\left(t\right)$. Once synchronization has been achieved,
    only the synchronization component of the control is manually disabled,
    while the orbital stabilization remains active. As a result, the robots
    continue to move on their stable orbital trajectories but no longer
    maintain time synchronization and gradually drift apart.

    \begin{figure}
    \begin{centering}
    \includegraphics[width=8.5cm]{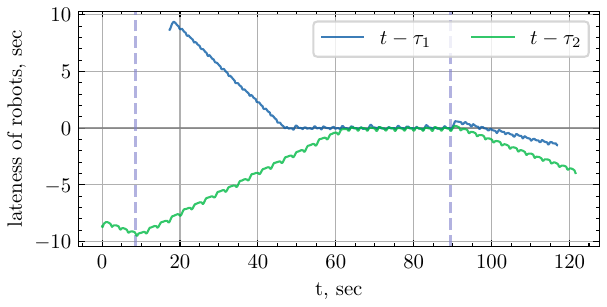}
    \par\end{centering}
    \caption{ Synchronization of two robots with respect to physical time. Dashed
    vertical lines indicate when synchronization is activated and deactivated.
    Once synchronization is turned off, the robots begin to drift apart.}\label{fig:sync-all-to-global-time}
    \end{figure}

  \subsection{Decentralized Synchronization of Six Robots}

    In the following experiments, we investigate decentralized synchronization
    schemes based on a heuristic modification of the control law~(\ref{eq:sliding-nonlin-control}).
    Specifically, we replace the reference physical time $t$ in the control
    law with a reference time $\tau_{*}$ computed from the self-times
    of other robots (either the self-time of a neighboring robot or the
    average over all active robots). Because the dynamics of $\tau_{*}$
    differ from those of the physical time $t$, this modification provides
    no formal guarantee of convergence. Nevertheless, we leverage the
    inherent robustness of the sliding-mode feedback to compensate for
    these dynamic differences. This setup allows us to experimentally
    assess whether a network of robots can achieve synchronization using
    only local time signals.

    The experiments are organized as follows. Six identical robots are
    launched sequentially with varying time intervals. After launch, the
    control system of the $i$-th robot starts publishing its self-time
    $\tau_{i}$ at 60 Hz and receives the self-times of all other active
    robots (i.e., those currently publishing). Thus, each controller maintains
    an up-to-date list of active robots with their self-times. Depending
    on the selected synchronization scheme, the controller computes the
    reference time $\tau_{*}$ and the corresponding lateness $h_{i}=\tau_{*}-\tau_{i}$
    which is then used in the control law~(\ref{eq:sliding-nonlin-control}).
    Once the robots have synchronized, they are manually switched off
    one by one; the network automatically reconfigures, and each controller
    removes inactive robots from its list. The remaining group maintains
    synchronization, highlighting the decentralized nature of the proposed
    scheme.

    We tested several synchronization schemes previously studied in the
    context of coupled harmonic oscillators~\cite{Ren-2008}. In the
    first scheme, each robot references only its immediate neighbor; the
    lateness of the $i$-th robot is then given by
    \begin{equation}
    h_{i}\equiv\tau_{1+i\,\mathrm{mod}\,N}-\tau_{i},\quad i=1..N,\label{eq:lateness-one-neighbor}
    \end{equation}
    so that the robot network forms a directed graph with a ring topology.
    The experimental results for this scheme are shown in Fig.~\ref{fig:one-neighbor-sync}.
    For each active robot, the graph plots its lateness relative to the
    average self-time of all active robots. Vertical lines indicate the
    instants at which individual robots are launched or stopped. As observed,
    the group achieves synchronization within a finite time despite sequential
    activations and deactivations.

    \begin{figure}
    \begin{centering}
    \includegraphics[width=8.5cm]{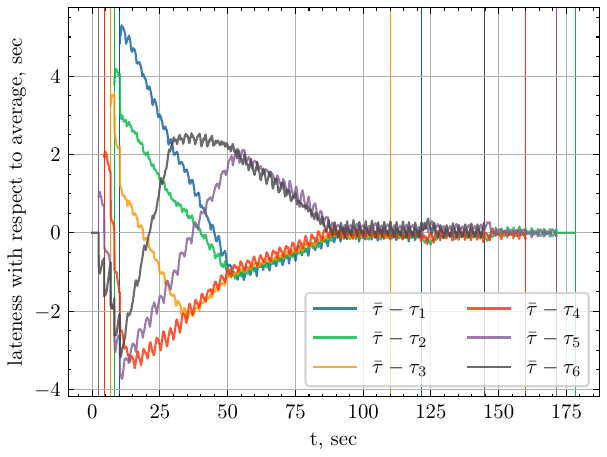}
    \par\end{centering}
    \caption{ Synchronization of six robots using neighbor-based references. Each
    robot references its immediate neighbor, forming a directed ring topology.
    The plot shows the lateness of each robot relative to the average
    of self-times of all active robots $\bar{\tau}\equiv\frac{1}{N}\sum_{i}\tau_{i}$.
    Vertical lines indicate when robots were launched or stopped. Self-times
    converge to the average, demonstrating synchronization across the
    group.}\label{fig:one-neighbor-sync}
    \end{figure}

    In the second scheme, the robots communicate over a complete directed
    graph, so that each controller has access to the self-times of all
    active robots. The reference time is computed as the average self-time,
    \[
    \bar{\tau}\equiv\frac{1}{N}\sum_{i}\tau_{i},\quad h_{i}\equiv\bar{\tau}-\tau_{i},,\quad i=1..N.
    \]
    The experimental results for this scheme are presented in Fig.~\ref{fig:avg-sync},
    and a video of the experiment is available at~\url{https://youtu.be/WtvYQWtRSiw}.

    \begin{figure}
    \begin{centering}
    \includegraphics[width=8.5cm]{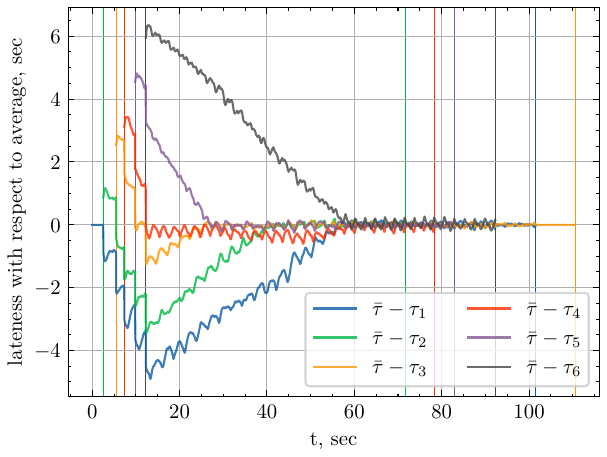}
    \par\end{centering}
    \caption{ Synchronization of six robots using average-based references. Each
    robot references the average of self-times of all active robots $\bar{\tau}\equiv\frac{1}{N}\sum_{i}\tau_{i}$,
    ensuring that all robots converge to a common time and demonstrating
    synchronization across the group.}\label{fig:avg-sync}
    \end{figure}

  \section{CONCLUDING REMARKS}
  \label{sec:Concluding-Remarks}

    We have presented a modification of the transverse--linearization
    approach that enables asymptotic stabilization of a reference periodic
    orbit. The method combines the transverse framework with a sliding--mode
    technique, providing finite-time stabilization of the time-desynchronization
    dynamics. The theoretical framework has been validated through a series
    of experiments, which demonstrate the effectiveness of the proposed
    method not only asymptotic trajectory stabilization of a single robot,
    but also for constructing decentralized control systems for a group
    of robots.

    Although the approach has been formulated for robots with an $n$-dimensional
    phase space and a single control input, it can be extended naturally
    to multi-input systems. We have omitted the proof of stability for
    the closed-loop nonlinear system under the presented control law,
    as a detailed analysis is beyond the scope of a conference paper;
    this proof will be provided in a separate publication. Additionally,
    a complete proof of stability for the proposed decentralized control
    schemes is also left for future work.

\end{document}